\documentclass[runningheads]{llncs}
\usepackage{amsfonts,amssymb,amsmath,color}
\usepackage[vlined,ruled,linesnumbered]{lib/algorithm2e}
\usepackage{lib/algpseudocode,lib/algcompatible, lib/algo}
\usepackage{graphicx}

\begin{document}
\title{Bisimulations for Neural Network Reduction}
%
%
\author{Pavithra Prabhakar}
\authorrunning{P. Prabhakar}
%
\institute{Kansas State University, Manhattan, KS, 66506\\
\email{pprabhakar@ksu.edu}}
\maketitle              
\begin{abstract}
We present a notion of bisimulation that induces a reduced network
which is semantically equivalent to the given neural network.
We provide a minimization algorithm to construct the smallest
bisimulation equivalent network.
Reductions that construct bisimulation equivalent neural networks are
limited in the scale of reduction.
We present an approximate notion of bisimulation that provides
semantic closeness, rather than, semantic equivalence, and quantify
semantic deviation between the neural networks that are approximately
bisimilar. 
The latter provides a trade-off between the amount of reduction and
deviations in the semantics.
\end{abstract}

\newcommand{\PP}[1] {{\color{red} #1}}
\newcommand{\naturals}{\ensuremath{\mathbb{N}}}
\newcommand{\reals}{\ensuremath{\mathbb{R}}}

\newcommand{\finer}{\ensuremath{\preceq}}
\renewcommand{\O}{\ensuremath{\textit{O}}}

\renewcommand{\sb}[1]{\ensuremath{\textit{SplitActBias}(#1)}}
\renewcommand{\sp}[2]{\ensuremath{\textit{SplitPre}(#1, #2)}}
\newcommand{\dist}[2]{\ensuremath{|#1 - #2|}}
\newcommand{\alp}[1]{\ensuremath{\alpha(#1)}}
\newcommand{\alpa}[2]{\ensuremath{\alpha^{#2}(#1)}}
\newcommand{\absf}[3]{\ensuremath{\alpha_{(#1,#2)}^{#3}}}
\newcommand{\Ah}[1]{\ensuremath{\widehat{A}_{#1}}}
\newcommand{\Lh}[1]{\ensuremath{\widehat{L}_{#1}}}
\newcommand{\bh}[1]{\ensuremath{\widehat{b}_{#1}}}
\newcommand{\Wh}[1]{\ensuremath{\widehat{W}_{#1}}}
\newcommand{\Sh}{\ensuremath{\widehat{\mathcal{S}}}}

\newcommand{\absa}[3]{\ensuremath{#1/_{#3}#2}}
\newcommand{\abs}[2]{\ensuremath{#1/#2}}
\newcommand{\ps}[2]{\ensuremath{\textit{PreSum}_{#1}^{#2}}}


\newcommand{\norm}[1]{\ensuremath{|#1|}}
\renewcommand{\a}{\ensuremath{\textit{a}}}
\newcommand{\bc}{\ensuremath{\textit{b}}}
\renewcommand{\P}{\ensuremath{\mathcal{P}}}
\newcommand{\comp}{\ensuremath{\circ}}
\newcommand{\val}[1]{\ensuremath{\textit{Val}(#1)}}
\newcommand{\sem}[1]{\ensuremath{[\![#1]\!]}}
\newcommand{\seml}[2]{\ensuremath{[\![#1]\!]_{#2}}}
\newcommand{\semc}[2]{\ensuremath{[\![#1]\!]^{#2}}}
\newcommand{\actions}{\ensuremath{\mathcal{A}ct}}
\newcommand{\A}[1]{\ensuremath{A_{#1}}}
\renewcommand{\L}[1]{\ensuremath{L_{#1}}}
\renewcommand{\b}[1]{\ensuremath{b_{#1}}}
\newcommand{\W}[1]{\ensuremath{W_{#1}}}
\renewcommand{\S}{\ensuremath{\mathcal{S}}}
\newcommand{\s}{\ensuremath{\textit{s}}}
\renewcommand{\v}{\ensuremath{\textit{v}}}
\newcommand{\lc}[1]{\ensuremath{L(#1)}}
\renewcommand{\k}{\ensuremath{k}}
\newcommand{\N}{\ensuremath{\mathcal{N}}}
\newcommand{\Ns}{\ensuremath{\mathcal{N}^*}}
\newcommand{\ainn}{\ensuremath{\text{AINN}}}
\newcommand{\nn}{\ensuremath{\text{NN}}}
\newcommand{\infn}[1]{\ensuremath{|\!| #1 |\!|_\infty}}

\section{Introduction}
Neural networks ($\nn$) with small size are conducive for both
automated analysis and explainability.
Rigorous automated analysis using formal methods has gained momentum
in recent years owing to the safety-criticality of the application
domains in which $\nn$ are
deployed~\cite{bunel,smtnn1,reluplex,taylor-nn3,taylor-survey,survey2019}. 
For instance, $\nn$ are an integral part of control, perception and
guidance of autonomous vehicles.
However, the scalability of these analysis techniques, for instance,
for computing output range for safety analysis~\cite{range,reluplex}, is
limited by the large size of the neural networks encountered and the 
computational complexity due to the presence of non-linear activation
functions. 
In this paper, we borrow ideas from formal methods to design novel
network reduction techniques with formal relations between the given
and the reduced networks, that can be applied to reduce the
verification time.
It can also potentially impact explainability by presenting to the
user a smaller network with guaranteed bounds on the deviation from
the original network.

Bisimulation~\cite{milner89} is a classical notion of equivalence between
systems in process algebra that guarantees that processes that are
bisimilar satisfy the same set of properties specified in certain
branching time logics~\cite{baier08}.
A bisimulation is an equivalence relation on the states of a system that
requires similar behaviors with respect to one step of computation,
which then inductively guarantees global behavioral equivalence.
Bisimulation algorithm~\cite{baier08} allows one to construct the smallest systems,
bisimulation quotients, that are bisimilar to a given (finite state)
system.

Our first result consists of a definition of bisimulation for neural
networks, namely, $\nn$-bisimulation, that defines a notion of
equivalence between neural networks.
The challenge arises from the fact that neural networks semantically have
multiple parallel threads of computation that are both branching and
merging at each step of computation.
We observe that the global equivalence can be established by imposing
a \emph{one-step backward pre-sum equivalence}, wherein we require two
nodes that belong to the same class to agree on the biases, the
activation functions, and the pre-sums, wherein a pre-sum corresponds
to the sum of the weights on the incoming edges from a given
equivalence class. 
Our notion resembles that of probabilistic
bisimulation~\cite{bisim-larsen}, however, our notion is based on pre-sum
equivalences rather than post-sum equivalences.
We define a quotienting operation on an $\nn$ with respect to a
bisimulation that yields a smaller network which is
input-output equivalent to the given network.
We also show that there exists a coarsest bisimulation which yields
the smallest neural network with respect to the quotienting operation.
We provide a minimization algorithm that outputs this smallest neural network.

The notion of bisimulation can be stringent, since, it preserves the
exact input-output relation. 
It has been observed, for instance, in the context of control systems,
that a strict notion of equivalence, such as, bisimulation, does not
allow for drastic reduction in state-space, thereby, motivating the notion of
approximate bisimulation. 
Approximate bisimulations~\cite{girard07,girard08} have a notion of distance
between states, and allow a bounded $\epsilon$ deviation between the
executions of the systems in each step.
The notion of approximate bisimulation was successfully used to
construct smaller systems in the setting of dynamical systems and control
synthesis~\cite{girardhscc08}.

We extend the notion of $\nn$-bisimulation to an approximate notion,
wherein we require nodes belonging to the same class to have bounded
deviation, $\epsilon$, in the biases and the pre-sums.
The quotienting operation no more results in a unique reduced network,
but a set of reduced networks.
Moreover, these reduced networks may not have the same input-output
relation as the given neural network.
However, we provide a bound on the deviation in the semantics of two
approximately bisimilar $\nn$s.
It gives rise to a nice trade-off between the amount of reduction and
the deviation in the semantics, that translates to a trade-off in the
precision and verification time in an approximation based verification
scheme.


\paragraph{Related work.}
  Neural network reduction techniques have been explored in different
  contexts.
  There is extensive literature on compression techniques, see, for instance,
  surveys on network compression~\cite{neural-compression-1,neural-compression-2}.  
However, these techniques typically do not provide formal guarantees
on the relation between the given and reduced systems.
Abstraction techniques~\cite{neurips19,neural-sas20,neural-cav20}
computing over-approximations of the input-output relations have been
explored in several works, however, they use slightly different kinds
of networks such as interval neural networks and abstract neural
networks, or are limited to certain kinds of activation functions such as ReLU.
  Notions of bisimulation for DNNs have not been explored much in the
  literature. Equivalence between DNNs is
  explored~\cite{deepabstract}, however, the work is restricted to
  ReLU functions and does not consider approximate notions.

\section{Preliminaries}
Let $[k]$ denote the set $\{0, 1, \cdots, k\}$ and $(k]$ the set $\{1,
2, \cdots, k\}$.
Let $\reals$ denote the set of real numbers.
We use $\norm{x}$ to denote the absolute value of $x \in \reals$.
Given a set $A$, we use $\norm{A}$ to denote the number of elements of
$A$.
Given a function $f: A \to \reals$, we define the infinity norm of $f$
to be the supremum of the absolute values of elements in the range of
$f$, that is, $\infn{f} = \sup_{a \in A} \norm{f(a)}$.
Given functions $f: A \to B$ and $g: B \to C$, the composition of $f$
and $g$, $g \comp f: A \to C$, is given by, for all $a \in A$,  $g \comp f(a) = g(f(a))$.

\paragraph{Partitions.}
Given a set $\S$, a (finite) parition of $\S$ is a set $\P = \{\S_1, \cdots,
\S_n\}$, such that $\bigcup_i \S_i = \P$ and $\S_i \cap \S_j =
\emptyset$ for all $i \not= j$.
We refer to each element of a partition as a region or a group.
A partition $\P$ of $\S$ can be seen as an equivalence relation on
$\S$, given by the relation $s_1 \P s_2$ whenever $s_1$ and $s_2$
belong to the same group of the partition.
Given two partitions $\P$ and $\P'$ , we say that $\P$ is finer than
$\P'$ (or equivalently, $\P'$ is coarser than $\P$), denoted $\P
\finer \P'$, if for every $S \in \P$, there exists $S' \in \P'$ such
that $S \subseteq \S'$.


\section{Neural Networks}
In this section, we present the preliminaries regarding the neural
network.
Recall that a neural network ($\nn$) consists of a layered set of
nodes or neurons, including an input layer, an output layer and one or
more hidden layers.
Each node except those in the input layer are annotated with a bias
and an activation function, and there are weighted edges between nodes
of adjacent layers.
We capture these elements of a neural network using a tuple in the
following definition. 

\begin{definition} A neural network ($\nn$) is a tuple $\N = \big(\k,
  \actions, \{\S_i\}_{i \in [\k]},  \{\W{i}\}_{i \in (\k]},$
  $\{\b{i}\}_{i \in (\k]}, \{\A{i}\}_{i \in (\k]} \big)$,
where:
\begin{itemize}
\item $\k$ represents the number of layers (except the input layer);
\item $\actions$ is a set of activation functions;
\item for each $i \in [k]$, $\S_i$ is a set of nodes of layer $i$, we
  assume $\S_i \cap \S_j =\emptyset$ for $i \neq j$; 
\item for each $i \in (k]$, $\W{i}: \S_{i-1} \times \S_{i} \to \reals$ is the weight function
  that captures the weight on the edges between nodes at layer
  $i - 1$
  and $i$;
\item for each $i \in (k]$, $\b{i} : \S_i  \to \reals$ is the bias
  function that associates a bias with nodes of layer $i$;
\item for each $i \in (k]$, $\A{i} : \S_i \to \actions$ is an
  activation association function that associates an activation
  function with each node of layer $i$.
		\end{itemize}
	\end{definition}
$\S_0$ and $\S_{\k}$ are the set of nodes corresponding to the input
and output layers, respectively.
We will fix the $\nn$ $\N = \big(\k,
  \actions, \{\S_i\}_{i \in [\k]},  \{\W{i}\}_{i \in (\k]},
  \{\b{i}\}_{i \in (\k]}, \{\A{i}\}_{i \in (\k]} \big)$ for the rest
  of the paper.

\begin{figure}
  \begin{center}
    \includegraphics[scale=0.4]{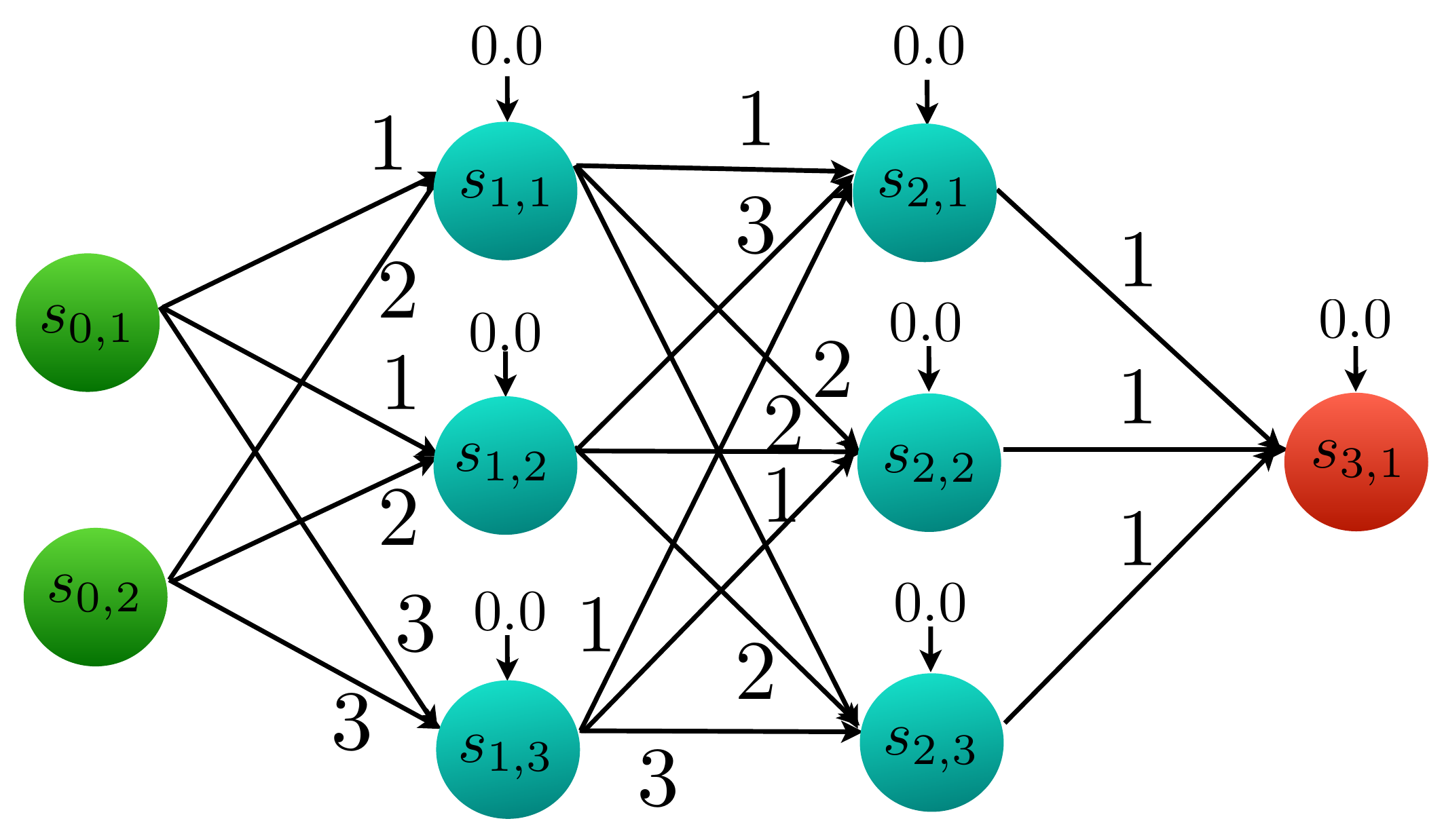}
    \caption{Neural Network $\N$}
    \label{fig:bisim}
    \end{center}
\end{figure}

\begin{example}
The neural network $\N$ shown in Figure \ref{fig:bisim} consists of
an input layer with $2$ nodes, $2$ hidden layers with $3$ nodes each,
and an output layer.
The weights on the edges are shown, for instance, $\W{2}(\s_{1, 2},
\s_{2,2}) = 2$. The biases are all $0$s and the activation
functions are all ReLUs (not shown).
\end{example}

In the sequel, the central notion to the definition of bisimulation
will be the total weight on the incoming edges for a node  $\s'$ of the
$i$-th layer from a set of nodes $\S$ of the $i-1$-st layer.
We will capture this using the notion of a pre-sum, denoted
$\ps{i}{\N}(\S, \s')$.
For instance, $\ps{2}{\N}(\{\s_{1,1}, \s_{1,2}\}, \s_{2, 2}) = 2 + 2 = 4$.
\begin{definition}
Given a set $\S \subseteq \S_{i-1}$ and $\s' \in \S_{i}$,  we define
$\ps{i}{\N}(\S, \s') = \sum_{\s \in \S} \W{i}(\s, \s')$.
\end{definition}

Next, we capture the operational behavior of a neural network.
A valuation $\v$ for the $i$-th layer of $\N$ refers to an assignment of real-values to all the
nodes in $\S_i$, that is, $\v: \S_i \to \reals$.
Let $\val{\S_i}$ denote the set of all valuations for the $i$-th
  layer of $\N$.
  By the operational semantics of $\N$, we mean the assignments for
  all the layers of $\N$, that are obtained from an assignment for the
  input layer.
  We define $\seml{\N}{i}(v)$, which given a valuation $v$ for layer
  $i-1$, returns the corresponding valuation for layer $i$ according
  to the semantics of $\N$.
  The valuation for the output layer of $\N$ is then obtained by
  the composition of the functions $\seml{\N}{i}$.

  \begin{definition} The semantics of the $i$-the layer is the
    function $\seml{\N}{i}: \val{\S_{i-1}} \to  \val{\S_{i}}$, where
    for any $\v \in \val{\S_{i-1}}$,  $\seml{\N}{i}(\v) = \v'$, given by
\[\forall \s' \in \S_{i}, \v'(\s') = \A{i}(\s')\big(\sum_{\s \in \S_{i-1}} \W{i}(\s,\s')\v(\s)
  +\b{i}(\s')\big).\]
\end{definition}
To capture the input-output semantics, we compose these one layer
semantics.
More precisely, we define $\semc{\N}{i}$ to be the composition of the
first $i$ layers, that is, $\semc{\N}{i}(\v)$ provides the valuation
of the $i$-th layer given $\v$ as input. It is defined inductively as:
\[\semc{\N}{1} = \seml{\N}{1}\]
\[\forall i \in (k]\backslash \{1\}, \semc{\N}{i} = \seml{\N}{i} \comp \semc{\N}{i-1}\]

\begin{definition}
The input-output semantic function, represented by $\sem{\N}:  \val{\S_{0}} \to
\val{\S_{k}}$, is defined as:
\[\sem{\N} = \semc{\N}{k}\]
\end{definition}

The notion of bisimulation requires the notion of a partition of the
nodes of $\N$.
We define a partition on $\N$ as an indexed set of partitions each
corresponding to a layer. 
\begin{definition}
A partition of an $\nn \ \N$  is an indexed set $\P =\{\P_i\}_{i \in
  [\k]}$, where for every $i$, $\P_i$ is a partition of $\S_i$. 
\end{definition}



\paragraph{A note on Lipschitz Continuity.}
A function $f: \reals^m \to \reals^n$ is said to be Lipschitz continuous
if there exists a constant $\lc{f}$, referred to as Lipschitz constant
for $f$, such that for all $x, y \in \reals^m$,
\[\infn{f(x)-f(y)} \leq \lc{f} \infn{x-y}.\]
Several activation functions including ReLU, Leaky ReLU, SoftPlus,
Tanh, Sigmoid, ArcTan and Softsign are known to be $1$-Lipschitz
continuous~\cite{nips18-lipschitz}, that is, satisfy the above
constraint with $\lc{f} = 1$.
In fact, the function $\sem{\N}$ is itself Lipschitz continuous, when the
activation functions are Lipschitz continuous~\cite{nips18-lipschitz}.
We will use $\lc{\N}$ to denote an upper bound on $\lc{\semc{\N}{i}}$
over all $i$.
Hence, given an input $\v$, we know that $\infn{\semc{\N}{i}(\v)} \leq
\lc{\N} \infn{\v}$.

\section{NN-bisimulation and Semantic Equivalence}
In this section, we define a notion of bisimulation on neural
networks, which induces a reduced system that is equivalent to the
given network.
A partition of $\N$ is an $\nn$-bisimulation if the biases and
activation functions associated with the nodes in any region are the
same, and the pre-sums of nodes in any region with respect to any
region of the previous layer are the same.  
\begin{definition}
An $\nn$-bisimulation for $\N$ is a partition $\P =
\{\P_i\}_{i\in [k]}$ such that  for all $i \in (k]$, $\S \in \P_{i-1}$
and $\s'_1,  \s'_2 \in \S_i$ with $\s'_1 \P_i \s'_2$, the following hold:
\begin{enumerate}
\item
$\A{i}(\s'_1) = \A{i}(\s'_2)$,
\item
$\b{i}(\s'_1) = \b{i}(\s'_2)$, and
\item
  $\ps{i}{\N}(\S, \s'_1) =  \ps{i}{\N}(\S, \s'_2)$.
  \end{enumerate}
\end{definition}
Our notion is inspired by the well-known notion of probabilistic
bisimulation~\cite{bisim-larsen}, where post-sums are used instead of
pre-sums to characterize which nodes have the same branching
structure.
Though neural networks consist of branching in both forward and backward
directions, surprisingly, just pre-sum equivalence suffices to
guarantee input-output relation equivalence.

Bisimulation naturally induces a reduced system, which corresponds to
merging the nodes in a group of the partition, and choosing a
representative node from the group to assign the activation functions,
biases and pre-sums.
We represent the reduced system obtained by taking the quotient of $\N$ with
respect to a bisimulation $\P$ as $\abs{\N}{\P}$. 
\begin{definition}
  \label{def:red}
Given an $\nn$-bisimulation $\P$ for $\N$, the reduced system
$\abs{\N}{\P} = \big(\k,  \actions,
\{{\Sh}_i\}_{i \in [\k]},  \{\Wh{i}\}_{i \in (\k]}, \{\bh{i}\}_{i
  \in (\k]}, \{\Ah{i}\}_{i \in (\k]} \big)$, where:
  \begin{enumerate}
		\item $\forall i \in [\k], \Sh_i = \P_i$;
		\item $\forall i \in (\k], \widehat{\s} \in
                        \Sh_{i-1}, \widehat{\s}' \in \Sh_i,
                        \Wh{i}(\widehat{\s}, \widehat{\s}') = 
                        \ps{i}{\N}(\widehat{\s}, \s')$ for some $\s' \in \widehat{\s}'$.
		\item $\forall i \in (\k], \widehat{\s}' \in
                        \Sh_{i}, \bh{i}(\widehat{\s}') = \b{i}(\s')$ for
                  some $\s' \in \widehat{\s}'$. 
		\item $\forall i \in (\k], \widehat{\s}' \in
                        \Sh_{i}, \Ah{i}(\widehat{\s}') = \A{i}(\s')$ for some
                        $\s' \in \widehat{s}'$. 
		\end{enumerate}
              \end{definition}
Note that though the definition depends on the choice of 
$\s'$, the reduced system is unique, since, from the definition of
$\nn$-bisimulation, the values of biases, activation functions and
pre-sums, corresponding to different choices of
$\s'$ within a group are the same.
We also use just bisimulation to refer to $\nn$-bisimulation.

In order to formally establish the connection between the $\nn$ $\N$
and its reduction $\abs{\N}{\P}$, we define a mapping from the
valuations of $\N$ to those of $\abs{\N}{\P}$, but only for certain
valuations that are consistent in that they map all the
related nodes in $\P$ to the same value.

\begin{definition}
A valuation $\v \in \val{\S_i}$ is $\P$-consistent, if for
all $\s_1, \s_2 \in \S_i$, if $\s_1 \P_i \s_2$, then $\v(\s_1) = \v(\s_2)$.
\end{definition}

Our first result is that a consistent input valuation leads to a
consistent output valuation, when $\P$ is a bisimulation.
We show this for a particular layer; the extension to the whole
network follows from a simple inductive reasoning.

\begin{lemma}
  \label{lem:cons}
Let $\P$ be a bisimulation on $\N$.
If $\v_1 \in \val{\S_{i-1}}$ is $\P$-consistent, then $\v_2 =
\seml{\N}{i}(\v_1)$ is $\P$-consistent.
\end{lemma}
\begin{proof}
Let $\s', \s'' \in \S_i$ such that $\s' \P_i \s''$.
We need to show that $\v_2(\s') = \v_2(\s'')$.
$\v_2(\s') = \A{i}(\s')(\sum_{\s \in \S_{i-1}} \W{i}(\s, \s') \v_1(\s) +
\b{i}(\s')) =\A{i}(\s') (\sum_{\S \in \P_{i-1}}\sum_{\s \in \S}
\W{i}(\s, \s')$ $ \v_1(\s) + \b{i}(\s'))$.
Since, $\v_1$ is $\P$-consistent, for each $\S$, we have a value
$\v_1^\S$ that all elements of $\S$ are mapped to by $\v_1$, that is, $\v_1^\S =
\v_1(\s)$ for all $\s \in \S$.
Replacing $\v_1(\s)$ for each
$\s$, by $\v_1^\S$, we obtain $\v_2(\s') = \A{i}(\s') (\sum_{\S \in \P_{i-1}}\sum_{\s \in \S}
\W{i}(\s, \s') \v_1^\S + \b{i}(\s'))$.
From the definition of pre-sum, we can replace $\sum_{\s \in \S}
\W{i}(\s, \s')$ by $\ps{i}{\N}(\S, \s')$, which is also equal to
$\ps{i}{\N}(\S,\s'')$ from the definition of bisimulation, since $\P$
is a bisimulation and $\s' \P_i \s''$.
Also, $\A{i}(\s') = \A{i}(\s'')$ and $\b{i}(\s') = \b{i}(\s'')$.
So, we
obtain,
$\v_2(\s') = \A{i}(\s'') (\sum_{\S \in P_{i-1}} \ps{i}{\N}(\S,\s'')
\v_1^\S + \b{i}(\s''))$.
Expanding back $\ps{i}{\N}(\S,\s'')$, and $\v_1^\S = \v_1(\s)$  for
all $\s \in \S$, we obtain
$\v_2(\s') = \A{i}(\s'') (\sum_{\S \in P_{i-1}} \sum_{\s \in \S}
\W{i}(\s, \s'') \v_1(\s) + \b{i}(\s'')) = \v_2(\s'')$.
\end{proof}

Note that if we do not group together the nodes in the input and
output layers, there is a bijection between $\S_0$ and $\Sh_0$ and
$\S_k$ and $\Sh_k$, and hence, a bijection between their valuations.
We will show that both $\N$ and $\abs{\N}{\P}$ have the ``same''
input-output relation modulo the bijection between their nodes.
First, we define a formal relation between $\P$-consistent valuations
of $\N$ and valuations of $\abs{\N}{\P}$.

\begin{definition}
Let $\P$ be a bisimulation on $\N$, and $\v \in \val{\S_i}$ be a
$\P$-consistent valuation. 
The abstraction of $\v$, denoted, 
$\alp{\v}_{\N, \P} \in \val{\Sh_i}$, is defined as, for every
$\hat{\s} \in \Sh_i$,  $\alp{\v}_{\N,
  \P}(\hat{\s}) = \v(\s)$ for some $\s \in \hat{\s}$.
\end{definition}
Note that $\alp{\v}_{\N, \P}$ is well defined, since, from the
$\P$-consistency of $\v$, $\v(\s)$ is the same for any choice of $\s
\in \hat{\s}$.
When $\N$ and $\P$ are clear from the context, we will drop the
subscript and write $\alp{\v}_{\N, \P}$ as just $\alp{\v}$.
The next result states that the output of the $i$-th layer of
$\abs{\N}{\P}$ with the abstraction of a $\P$-consistent valuation
$\v$ of the $i-1$-st layer of $\N$ as input,
results in a valuation that is the abstraction of the output of the
$i$-th layer of $\N$ on input $\v$.
In other words, it says that propagating a valuation for one-step in
$\N$ is the same as propagating its abstraction in $\abs{\N}{\P}$.

\begin{lemma}
  \label{lem:abs}
Let $\P$ be a bisimulation on $\N$, and $\v \in \val{\S_i}$ be
$\P$-consistent.
Then, $\alp{\seml{\N}{i}(\v)} = \seml{\abs{\N}{\P}}{i}(\alp{\v})$.
\end{lemma}
\begin{proof}
From Lemma \ref{lem:cons}, we know that $\v' = \seml{\N}{i}(\v)$ is
$\P$-consistent.
Hence, for any $\hat{\s}' \in \Sh_i$, $\alp{\v'}(\hat{\s}') = \v'(\s')$ for
some (any) $\s' \in \hat{\s}'$.
Let us fix $\s' \in \hat{\s}'$.
\[\alp{\v'}(\hat{\s}') = \v'(\s') =  \A{i}(\s')\big(\sum_{\s \in \S_{i-1}} \W{i}(\s,\s')\v(\s)
  +\b{i}(s')\big)\]
from the semantics of $\N$.
Further,
\[\sum_{\s \in \S_{i-1}} \W{i}(\s,\s')\v(\s) = \sum_{\S \in \P_{i-1}}
  \sum_{\s \in \S} \W{i}(\s,\s')\v(\s)\]

From $\P$-consistency of $\v$, $\v(\s) = \alp{\v}(\S)$ for any $\s \in
\S$. Hence,
\[\sum_{\s \in \S} \W{i}(\s,\s')\v(\s) = \sum_{\s \in \S}
  \W{i}(\s,\s')\alp{\v}(\S)
  = \ps{i}{\N}(\S, \s') \alp{\v}(\S).\]
From the definition of $\abs{\N}{\P}$,
$\Ah{i}(\hat{\s}') = \A{i}(\s'), \Wh{i}(\S,
\hat{\s}') = \ps{i}{\N}(\S, \s'), \bh{i}(\hat{\s}') =
\b{i}(\s'),$
and $ \P_{i-1} = \Sh_{i-1}$.
From $\P$-consistency of $\v$, $\alp{\v}(\S) =
\v(\s)$ for any $\s \in \S$.
Therefore, for any $\hat{\s}' \in \Sh_i$, 
\[\alp{\seml{\N}{i}(\v)}(\hat{\s}') = 
  \alp{\v'}(\hat{\s}')\]\[  = \Ah{i}(\hat{\s}')\big(\sum_{\S \in \Sh_{i-1}} \Wh{i}(\S,
  \hat{\s}') \alp{\v}(\S) +\bh{i}(\hat{\s}')\big) =
\seml{\abs{\N}{\P}}{i}(\alp{\v}) (\hat{\s}').\]
\end{proof}

The following theorem follows by composing the results from Lemma
\ref{lem:abs} for the different layers.
\begin{theorem}
Given $\P$ a bisimulation on $\N$, and $\v \in \val{\S_0}$ that is
$\P$-consistent, we have $\alp{\sem{\N}(\v)} =
\sem{\abs{\N}{\P}}(\alp{\v})$.
\end{theorem}
\begin{proof}
We can show by induction on $i$ that $\alp{\semc{\N}{i}(\v)} =
\semc{\abs{\N}{\P}}{i}(\alp{\v})$.
  \end{proof}

\begin{figure}
  \begin{center}
    \includegraphics[scale=0.4]{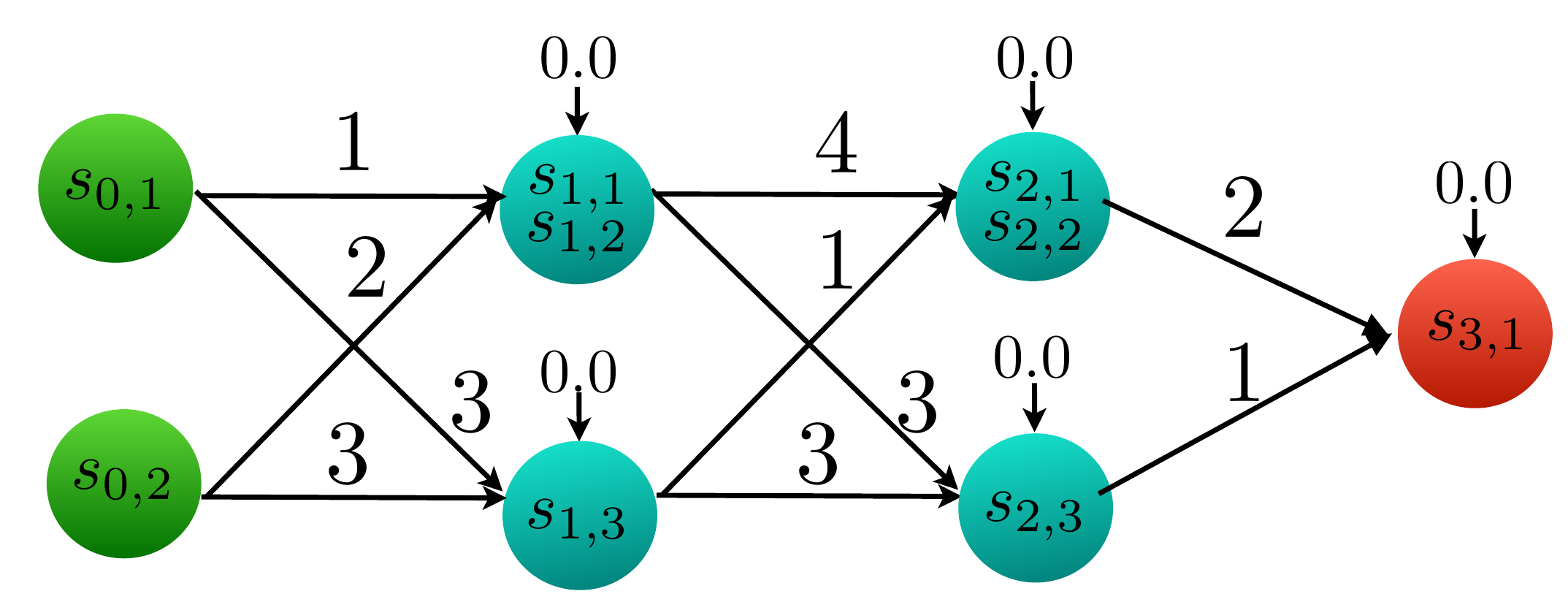}
    \caption{Reduced System $\abs{\N}{\P}$}
    \label{fig:mini}
    \end{center}
\end{figure}

\begin{example}
Consider a partition $\P$ for the $\nn$ $\N$ in Figure \ref{fig:bisim} where
each node appears as a region by itself except for the regions
$\S_1 = \{\s_{1,1}, \s_{1, 2}\}$, and $\S_2 = \{\s_{2,1},\s_{2,2}\}$.
We can verify that this is a bisimulation. For instance,
$\ps{2}{\N}(\S_1, \s_{2,1}) = 1 + 3$ and $\ps{2}{\N}(\S_1, \s_{2,2}) =
2 + 2$, which are the same.
The reduced system is given by the $\nn$ $\abs{\N}{\P}$ in Figure
\ref{fig:mini}.
Here, $\Wh{2}(\S_1, \S_2) = 4$.
 \end{example}
\section{$\delta$-$\nn$-bisimulation and Semantic Closeness}
$\nn$-bisimulation provides a foundation for reducing a neural network
while preserving the input-output relation.
However, existence of such bisimulations leading to equivalent reduced
networks with much fewer neurons is limited in that for many networks
no bisimulation quotient may lump together lot of nodes.
Hence, we relax the notion of bisimulation to an approximate notion
wherein we allow potentially large reductions, however, the reduced
systems may not be semantically equivalent, but only be semantically close
to the given neural network. 
We quantify the deviation of the reduced system in terms of the
``deviation'' of the approximate notion from the exact bisimulation.

The approximation notion of bisimulation we consider is inspired by
the notion of approximate bisimulation in the context of dynamical
systems~\cite{girard07,girard08}.
We essentially relax the requirement of the $\nn$-bisimulation that
the biases and pre-sums match by allowing them to be within a $\delta$.
This is formalized in the following definition.

\begin{definition}
A $\delta$-$\nn$-bisimulation for an $\nn$ $\N$ and $\delta \geq 0$ is
a partition $\P = \{\P_i\}_{i\in [\k]}$ such that  for all $i \in (k]$, $\S \in \P_{i-1}$
and $\s'_1,  \s'_2 \in \S_i$ with $\s'_1 \P_i \s'_2$, the following hold:
\begin{enumerate}
\item
$\A{i}(\s'_1) = \A{i}(\s'_2)$,
\item
$\dist{\b{i}(\s'_1)}{\b{i}(\s'_2)} \leq \delta$, and
\item
$\dist{\ps{i}{\N}(\S, \s'_1)}{\ps{i}{\N}(\S, \s'_2)} \leq \delta$.
 \end{enumerate}
\end{definition}

We will also use $\delta$-bisimulation to refer to $\delta$-$\nn$-bisimulation. 
The reduced system can be constructed similar to that for
$\nn$-bisimulation.
However, the choice of the nodes $\s' \in \hat{\s}'$  used to construct the weights and
biases of the reduced system could lead to different neural networks.
Hence, we obtain a finite set of possibilities for the reduced system
that we denote by $\absa{\N}{\P}{\delta}$.
\begin{figure}
  \begin{center}
    \includegraphics[scale=0.3]{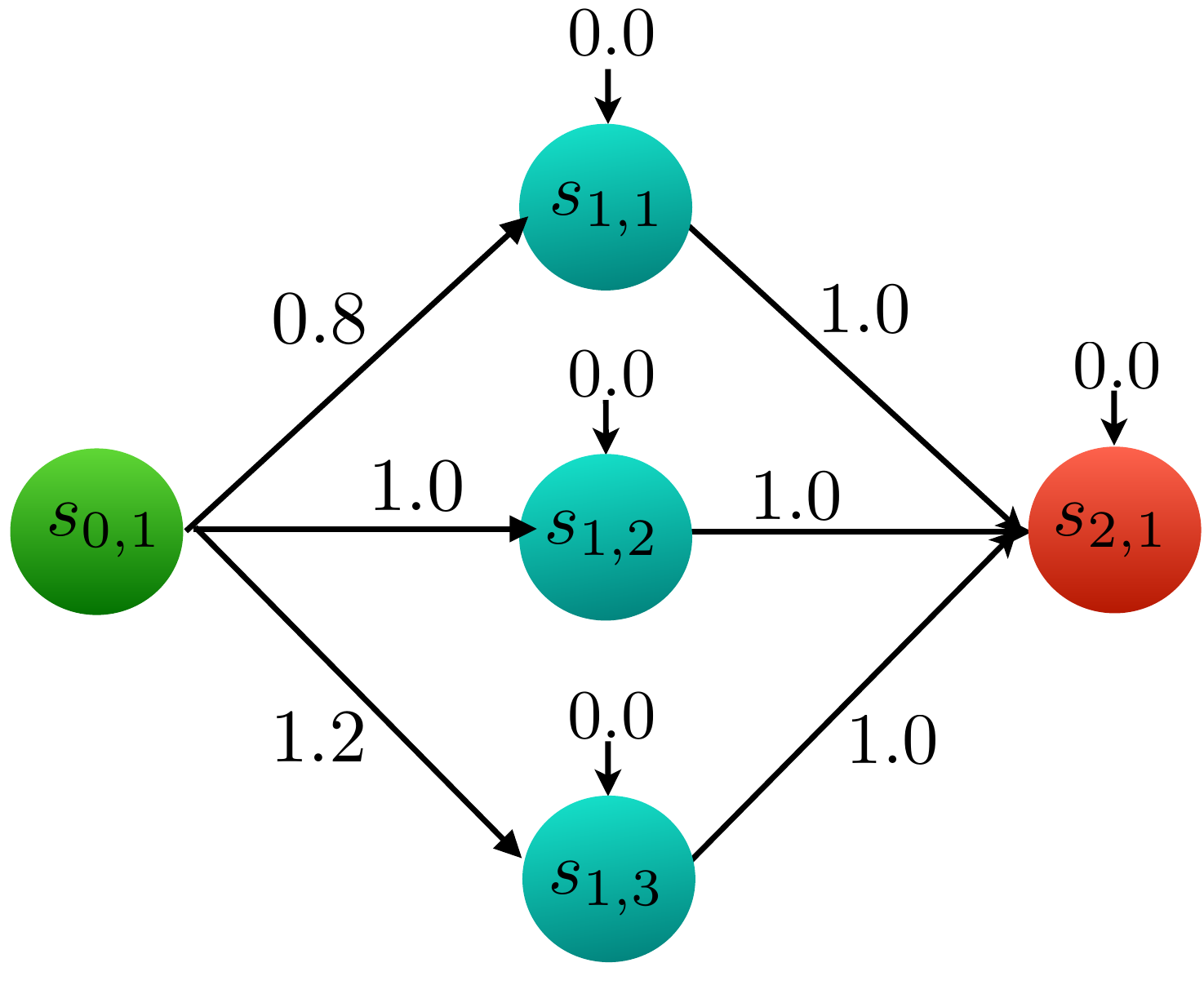}
    \includegraphics[scale=0.3]{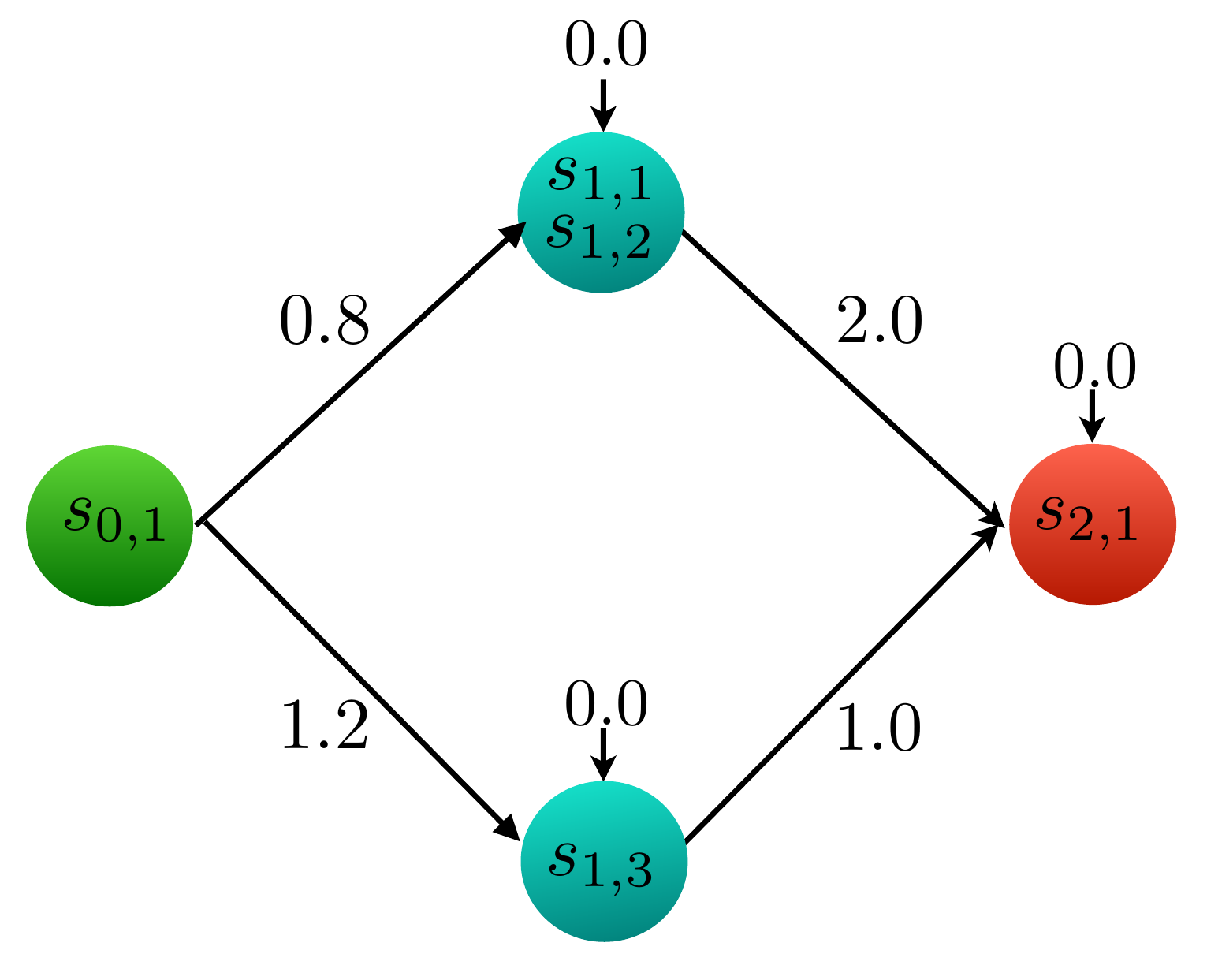}
        \includegraphics[scale=0.3]{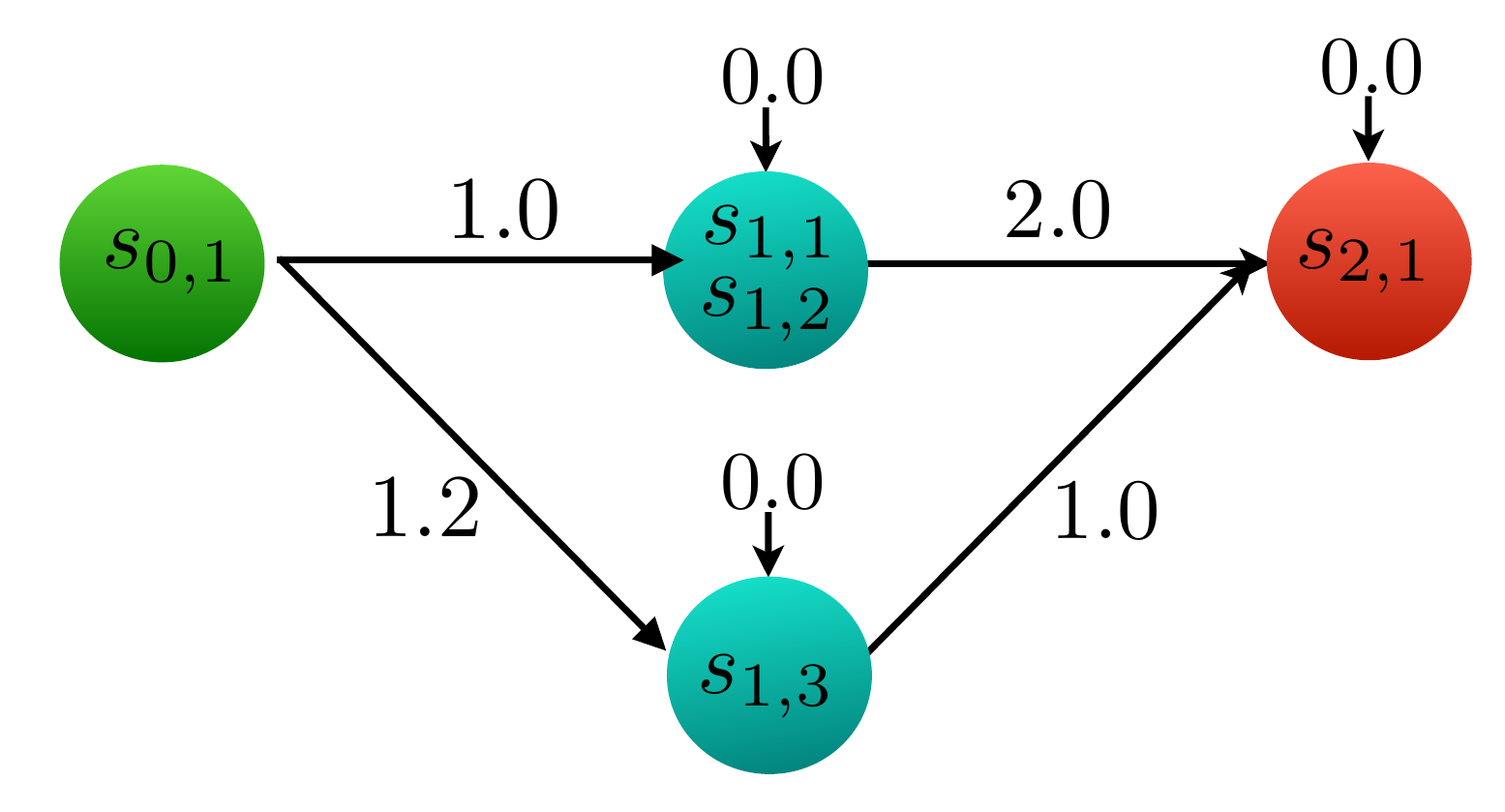}
    \caption{Illustration of $\absa{\Ns}{\P}{\delta}$ on $\nn$ $\Ns$}
\label{fig:approx}
    \end{center}
\end{figure}

\begin{example}
Consider the $\nn$ $\Ns$ in Figure \ref{fig:approx} (top left) and a partition
$\P = \{\P_i\}_i$, where $\P_0 = \{ \{ \s_{0,1}\}\}$, $\P_2 = \{ \{
\s_{2,1}\}\}$ and $\P_1 = \{ \{ \s_{1,1}, \s_{1,2} \}, \{
\s_{1,3}\}\}$, that is, $\P$ merges nodes $\s_{1,1}$ and $\s_{1,2}$.
Note that $\P$ is a $\delta$-bisimulation on $\Ns$ for $\delta = 0.2$.
For instance, $\ps{1}{\Ns}(\{ \s_{0,1}\},  \s_{1,1}) = 0.8$ and
$\ps{1}{\Ns}(\{ \s_{0,1}\},  \s_{1,2}) = 1.0$ whose difference is
$\leq 0.2 = \delta$.
$\absa{\N}{\P}{\delta}$ consists of $\Ns_1$ and $\Ns_2$ in Figure
\ref{fig:approx} (top right and bottom), one which is obtained by choosing
the pre-sum corresponding to $\s_{1,1}$ and other by choosing the
pre-sum corresponding to $\s_{1,2}$.
  \end{example}

Our objective is to give a bound on the deviation of the semantics of any $\N' \in
\absa{\N}{\P}{\delta}$ from that of $\N$.
We start by quantifying this deviation in one step of computation.
For that, we extend the notion of consistent valuations to an approximate
notion, wherein we require the valuations of related states to be
within a bound rather than match exactly.

\begin{definition}
A valuation $\v \in \val{\S_i}$ is $\epsilon, \P$-consistent, if for
all $\s_1, \s_2 \in \S_i$ with $\s_1 \P_i \s_2$, 
$\dist{\v(\s_1)}{\v(\s_2)} \leq \epsilon$.
\end{definition}

Our next step is to establish a relation between the valuation
propagation in $\N$ and any $\N' \in \absa{\N}{\P}{\delta}$ analogous
to Lemma \ref{lem:abs}.
First, we will need to relax the notion of the abstraction of a
valuation, however, unlike in the previous case, we obtain a set of
abstractions $\alpa{\v}{\epsilon}$.

\begin{definition}
Let $\P$ be a partition of $\N$, and $\v \in \val{\S_i}$.
The $\epsilon$-abstraction of $\v$, denoted, 
$\alpa{\v}{\epsilon}_{\N, \P}$, consists of $\hat{\v} \in 
\val{\P_i}$ such that  for all $\hat{\s} \in \P_i, \s \in \hat{\s}$, $\dist{\hat{\v}(\hat{\s})}{\v(\s)} \leq
\epsilon$.
\end{definition}
When $\N$ and $\P$ are clear from the context, we will drop the
subscript and write $\alpa{\v}{\epsilon}_{\N, \P}$ as just
$\alpa{\v}{\epsilon}$.
The next result states that the $\epsilon$-abstraction for any
$\epsilon,\P$-consistent valuation is non-empty.
\begin{proposition}
  Let $\v \in \val{\S_i}$ be an $\epsilon, \P$-consistent valuation.
  Then $\alpa{\v}{\epsilon}$ is non-empty.
\end{proposition}
\begin{proof}
Note that the valuation $\hat{\v}$, given by $\hat{\v}(\hat{\s}) =
\v(\s)$ for some $\s \in \hat{\s}$ gives a valuation in
$\alpa{\v}{\epsilon}_{\N, \P}$.
\end{proof}
The converse of the above theorem also holds with a slight
modification of the error.
\begin{proposition}
  \label{prop:bound}
  Let $\v \in \val{\S_i}$, such that $\alpa{\v}{\epsilon}_{\N, \P}$ is
  non-empty.
  Then,  $\v$ is a $2\epsilon, \P$-consistent valuation.
\end{proposition}
\begin{proof}
Note that there is some $\hat{\v}$, such that $\forall \hat{\s} \in
\P_i, \s \in \hat{\s}$, $\dist{\hat{\v}(\hat{\s})}{\v(\s)} \leq
\epsilon$. Then for all $\s, \s' \in \hat{\s}$, $\dist{\v(\s)}{\v(\s')} \leq
2\epsilon$.
  \end{proof}

Now, we give a bound on the deviation of the output of the $i$-th layer of
$\absa{\N}{\P}{\delta}$ from that of $\N$ in terms of the
deviation in their inputs.
Let $\lc{\A{i}} =  \max_{\s' \in   \S_i} \lc{\A{i}(\s')}$. 
\begin{lemma}
  \label{lem:approx}
Let $\P$ be a $\delta$-bisimulation on $\N$, and $\v \in \val{\S_{i-1}}$ be
$\epsilon, \P$-consistent.
Then, for every $\hat{\v} \in \alpa{\v}{\epsilon}$, and $\N' \in \absa{\N}{\P}{\delta}$,
\[\seml{\N'}{i}(\hat{\v}) \in \alpa{\seml{\N}{i}(\v)}{\epsilon'},\]
where $\epsilon' = \a_i \epsilon + \bc_i$, 
$\a_i = \lc{\A{i}} \norm{\S_{i-1}}\infn{\W{i}}$,  and $\bc_i = \lc{\A{i}}
(\norm{\P_{i-1}} \infn{\v} + 1) \delta$.
\end{lemma}
\begin{proof}
Let $\v' = \seml{\N}{i}(\v)$ and $\hat{\v}' =
\seml{\N'}{i}(\hat{\v})$.
We need to show that $\hat{\v}' \in \alpa{\v'}{\epsilon'}$.
Consider any $\hat{\s}' \in \Sh_i$ and $\s' \in \hat{\s}'$.
We need to show that $\dist{\hat{\v}'(\hat{\s}')}{\v'(\s')} \leq
\epsilon'$.

Since, $\hat{\v}' = \seml{\N'}{i}(\hat{\v})$, from the semantics of
$\N'$, we have
\[\hat{\v}'(\hat{\s}') = \Ah{i}(\hat{\s}')\big(\sum_{\hat{\s} \in
    \Sh_{i-1}} \Wh{i}(\hat{\s}, \hat{\s}') \hat{\v}(\hat{\s})
  +\bh{i}(\hat{\s}')\big),\]
and from the fact that $\N' \in \absa{\N}{\P}{\delta} $, we have 
$\Wh{i}(\hat{\s}, \hat{\s}') = \ps{i}{\N}(\hat{\s},\s'_1)$ for some $\s'_1 \in
\hat{\s}'$, $\bh{i}(\hat{\s}') = \b{i}(\s'_2)$ for some $\s'_2 \in
\hat{\s}'$.
Since $\P$ is a $\delta$-bisimulation,
$\dist{\ps{i}{\N}(\hat{\s},\s'_1)}{\ps{i}{\N}(\hat{\s},\s')} \leq
\delta$, $\dist{\b{i}(\s'_2)}{\b{i}(\s')}\leq \delta$, and $\Ah{i}(\hat{\s}') = \A{i}(\s')$.
Therefore,
\[\hat{\v}'(\hat{\s}') = \A{i}(\s')\big(\sum_{\hat{\s} \in
    \P_{i-1}}(\ps{i}{\N}(\hat{\s},\s') + \delta_{\hat{\s}})
 \hat{\v}(\hat{\s})+ \b{i}(\s') + \delta_{\s'}\big),\]
 \[= \A{i}(\s')\big(\sum_{\hat{\s} \in
    \P_{i-1}}\ps{i}{\N}(\hat{\s},\s')  \hat{\v}(\hat{\s})  +
  \b{i}(\s') + \epsilon_1 + \delta_{\s'}\big),\]
where $\epsilon_1 = \sum_{\hat{\s} \in
    \P_{i-1}} \delta_{\hat{\s}} \hat{\v}(\hat{\s})$ and $\delta_{\hat{\s}}, \delta_{\s'} \in [-\delta, \delta]$.
We will examine the terms in the above expression in more detail.

\[\sum_{\hat{\s} \in
    \P_{i-1}}\ps{i}{\N}(\hat{\s},\s')  \hat{\v}(\hat{\s})
= \sum_{\hat{\s} \in
    \P_{i-1}} [\sum_{\s \in \hat{\s}} \W{i}(\s,\s')  \hat{\v}(\hat{\s})]\]

(Further, since, $\hat{\v} \in \alpa{\v}{\epsilon}$, we have for any
$\s \in \hat{\s}$, 
$\dist{\hat{\v}(\hat{\s})}{\v(\s)} \leq \epsilon$.)
\[= \sum_{\hat{\s} \in
    \P_{i-1}} [\sum_{\s \in \hat{\s}} \W{i}(\s,\s')  (\v(\s) + \epsilon_\s)]
= \sum_{\s \in \S_{i-1}} \W{i}(\s,\s')  (\v(\s) + \epsilon_\s)]\]
\[= \sum_{\s \in \S_{i-1}} \W{i}(\s,\s')  \v(\s) + \sum_{\s \in
    \S_{i-1}} \W{i}(\s,\s') \epsilon_\s\]
Plugging the above into the expression for $\hat{\v}'(\hat{\s}')$, we
obtain

\[\hat{\v}'(\hat{\s}') = \A{i}(\s')\big(\sum_{\s \in \S_{i-1}} \W{i}(\s,\s')  \v(\s) +
  \b{i}(\s') + \epsilon_1 + \epsilon_2 + \delta_{\s'}\big)\]
where $\epsilon_2 = \sum_{\hat{\s} \in \P_{i-1}} \delta_{\hat{\s}}
\hat{\v}(\hat{\s})$.
Note that the expression for $\hat{\v}'(\hat{\s}')$ looks similar to
$\v'(\s') =\A{i}(\s')\big(\sum_{\s \in \S_{i-1}} \W{i}(\s,\s')  \v(\s)
+   \b{i}(\s') )$ except for the additional error terms $\epsilon_1 +
\epsilon_2 + \delta_{\s'}$.
From the Lipschitz continuity of $\A{i}(\s')$, we obtain
\[ \dist{\hat{\v}'(\hat{\s}')}{\v'(\s')} \leq \lc{\A{i}}(\s') (\norm{\epsilon_1 +
\epsilon_2 + \delta_{\s'}})\]


Note that $\lc{\A{i}}(\s') \leq \lc{\A{i}}$,
$\norm{\epsilon_1} = \norm{\sum_{\s \in \S_{i-1}} \W{i}(\s,\s')
  \epsilon_\s} \leq \norm{\S_{i-1}}\infn{\W{i}}\epsilon$,
$\norm{\epsilon_1} = \norm{\sum_{\hat{\s} \in \P_{i-1}} \delta_{\hat{\s}}
  \hat{\v}(\hat{\s})} \leq \norm{\P_{i-1}} \delta \infn{\v}$, and
$\norm{\delta_{\s'}} \leq \delta$.
Hence, 
\[\dist{\hat{\v}'(\hat{\s}')}{\v'(\s')} \leq \lc{\A{i}}(\s') (\norm{\epsilon_1 +
    \epsilon_2 + \delta_{\s'}}) \leq 
\lc{\A{i}} (\norm{\S_{i-1}}\infn{\W{i}}\epsilon +\norm{\P_{i-1}}
  \delta \infn{\v} + \delta)\]
\[= \lc{\A{i}} \norm{\S_{i-1}}\infn{\W{i}}\epsilon + \lc{\A{i}}
(\norm{\P_{i-1}} \infn{\v} + 1) \delta = \epsilon'\]
\end{proof}

Lemma \ref{lem:approx} provides a bound on the error propagation in one
step.
The next theorem provides a global bound on the deviation of the
output of the reduced system with respect to that of the given neural network.
Let $\lc{\A{ }} = \max_i \lc{\A{i}}$, $\norm{\P} = \max_i
\norm{\P_{i}}$,  $\infn{\W{ }} = \max_i\infn{\W{i}}$ and $\norm{\S} =
\norm{\max_i \S_{i}}$. 
\begin{theorem}
\label{thm:approx}
Let $\P$ be a $\delta$-bisimulation on $\N$, and $\v \in \val{\S_0}$
be $\epsilon,\P$-consistent.
Then, for every $\hat{\v} \in
\alpa{\v}{\epsilon}$, and $\N' \in \absa{\N}{\P}{\delta}$,
\[\sem{\N'}(\hat{\v}) \in
\alpa{\sem{\N}(\v)}{\epsilon''},\] where $\epsilon'' = [(2/a)^k - 1] \bc/
(2 \a - 1)$, $\a = \lc{\A{ }} \norm{\S} \infn{\W{ }}$, and $\bc =\lc{\A{ }}
(\norm{\P} \lc{\N} \infn{\v} + 1) \delta$.
\end{theorem}
\begin{proof}
Let us define:
\[\v_0 = \v, \hat{\v}_0 = \hat{\v}, \epsilon_0 = \epsilon'_0 = 0\]
and for all $i \in (k]$,
\[\v_i = \seml{\N}{i}(\v), \hat{\v}_i = \seml{\N'}{i}(\hat{\v}).\]
\[\epsilon'_i = \a \epsilon_{i-1} + \bc, \epsilon_i = 2 \epsilon'_i.\]
We will show by induction on $i$ that for all $i \in [k]$, $\v_i$ is $\epsilon_i,
\P$-consistent and $\hat{\v}_i \in \alpa{\v_i}{\epsilon'_i}$.

\noindent{\bf Base case:} Base case trivially holds from the
assumptions of the theorem statement.

\noindent{\bf Induction Step:}
For $i \in (k]$, we know from Lemma \ref{lem:approx}, that if 
$\v_{i-1} \in \val{\S_{i-1}}$ is $\epsilon_{i-1}, \P$-consistent and $\hat{\v}_{i-1}
\in \alpa{\v_{i-1}}{\epsilon_{i-1}}$, then $\hat{\v}_i =
\seml{\N'}{i}(\hat{\v}_{i-1}) \in
\alpa{\seml{\N}{i}(\v_{i-1})}{\epsilon'} =
\alpa{\v_i}{\epsilon'}$, where $\epsilon' = \a_i \epsilon_{i-1} +
\bc_i$.
\[\a_i = \lc{\A{i}} \norm{\S_{i-1}}\infn{\W{i}} \leq \lc{\A{ }}
  \norm{\S}\infn{\W{ }} = \a,\]
\[\bc_i = \lc{\A{i}}
(\norm{\P_{i-1}} \infn{\v_i} + 1) \delta \leq \lc{\A{ }}
(\norm{\P} \lc{\N}\infn{\v} + 1) \delta = \bc\]
Hence, $\epsilon \leq \epsilon'_i$ and  $\hat{\v}_i \in \alpa{\v_i}{\epsilon'_i}$.
Further from Proposition \ref{prop:bound}, we obtain  $\v_i$ is
$\epsilon_i, \P$-consistent.

We will show that $\epsilon'_k = \epsilon''$.
Unrolling the recursive equation, we obtain $\epsilon'_i = 2 \a
\epsilon'_{i-1} + \bc = (2 \a)^i \epsilon'_0 + [(2 \a)^{i-1} + \cdots +
1] \bc = [(2/a)^i - 1] \bc/ (2 \a - 1)$.
Hence,
\[\epsilon'_k = [(2/a)^k - 1] \bc/ (2 \a - 1) = \epsilon''\]
We finish the proof by noting that $\sem{\N'}(\hat{\v}) = \hat{\v}_k \in
\alpa{\v_k}{\epsilon'_k} = \alpa{\sem{\N}(\v)}{\epsilon''}$.
\end{proof}

\begin{remark}
Note that for $\delta = 0$, all the notions and results reduce to that
of $\nn$-bisimulation.
\end{remark}

\section{Minimization algorithm}
In this section, we show that there is a coarsest $\nn$-bisimulation
for a given $\nn$, that encompasses all other bisimulations.
This implies that the induced reduced network with respect to this
coarsest bisimulation is the smallest $\nn$-bisimulation equivalent network.
We will provide an algorithm that outputs the coarsest
$\nn$-bisimulation.

We note that a coarsest $\delta$-$\nn$-bisimulation
may not exist in general.
For instance, consider the $\nn$ $\Ns$ from Figure \ref{fig:approx},
along with the $0.2$-bisimulation $\P$ that induces the reduced
systems $\Ns_1$ and $\Ns_2$.
There is another $0.2$-bisimulation $\P'$ which is obtained by merging
$\s_{1,2}$ and $\s_{1,3}$ instead of $\s_{1,1}$ and $\s_{1,2}$ as in
$\P$.
Note that the reduced networks in $\absa{\Ns}{\P}{0.2}$ and
$\absa{\Ns}{\P'}{0.2}$ have the same size.
However, there is no $0.2$-bisimulation that is coarser than both $\P$
and $\P'$, since, that would require merging $\s_{1,1}, \s_{1,2}$ and
$\s_{1,3}$, which would violate the $0.2$ bound on the difference between
the pre-sums of $\s_{1,1}$ and $\s_{1, 3}$.


The broad algorithm for minimization consists of starting with a partition in which all
the nodes in a layer are merged together and then splitting them such
that the regions in the partition respect the activation functions,
biases and the pre-sums. 
We use the function $\sb{\S}$ in the algorithm that splits a set of nodes $\S$ into
maximal groups such that the elements in each group agree on the
activation functions and the biases. 
More precisely, $\sb{\S}$ takes $\S \subseteq \P_i$ as input and
returns a partition $\P_\S$ such that for all $\s_1, \s_2 \in \S$,
$\s_1 \P_\S \s_2$ if and only if $\A{i}(\s_1) = \A{i}(\s_2)$ and
$\b{i}(\s_1) = \b{i}(\s_2)$. 
Further, we split those regions that have nodes with inconsistent pre-sums.
Next, we define inconsistent pairs of regions with respect to pre-sums
and the corresponding splitting operations. 

\begin{definition}
Given a partition $\P = \{\P_i\}_i$ of $\nn$ $\N$, a region $\S' \in
\P_i$ is inconsistent in $\N$ with respect to $\S \in 
\P_{i-1}$, written $(\S', \S)$ inconsistent, if there exist $\s'_1,
\s'_2 \in \S'$, such that $\ps{i}{\N}(\S, \s'_1) \not= \ps{i}{\N}(\S,
\s'_2)$. 
  \end{definition}

The algorithm searches for inconsistent pairs $(\S', \S)$ and splits
$\S'$ into maximal groups such that all nodes in a group have the same
pre-sum with respect to $\S$.
More precisely, $\sp{\S'}{\S}$ takes $\S'$ and $\S$ as input and
returns a partition $\P'$ of $\S'$ such that $\ps{i}{\N}(\S, \s'_1) =
\ps{i}{\N}(\S, \s'_1)$ if and only if $\s'_1 \P' \s'_2$.
  
\begin{algorithm}[h]
\vspace{2 mm}
\KwIn{A $\nn$ $\N$}
\KwOut{Coarsest Bisimulation $\P$, and Minimized $\nn$ $\abs{\N}{\P}$}
\vspace{2 mm}
\caption{\textbf{MinNN:} Minimization Algorithm}
\label{algo:mini}
\SetKw{KwGoTo}{goto}
\Begin{
  $\P  = \{\S_0\}$\\
 \For{$i \in (k]$}
  {
    $\P = \P \cup \sb{\S_i}$
  }
  \While{Exists $\S, \S' \in \P$, such that $(\S, \S')$ inconsistent}
  {
    $\P = \P \backslash \{\S'\} \cup \sp{\S'}{\S}$
    }
 				
\textbf{return} Return $\P$ and $\abs{\N}{\P}$
}
\end{algorithm}

Next, we show that Algorithm \ref{algo:mini} returns the coarsest
bisimulation, and hence, the reduced network is the smallest
bisimulation equivalent network.

\begin{definition}
A partition $\P$ of $\N$ is the coarsest bisimulation, if it is an
$\nn$ bisimulation and it is coarser than every $\nn$-bisimulation
$\P'$ of $\N$. 
  \end{definition}

\begin{theorem}
Algorithm \ref{algo:mini} terminates and returns the coarsest
bisimulation $\P$ of $\N$.
\end{theorem}
\begin{proof}
Termination of the algorithm is straightforward, since, if there
exists an inconsistent pair $(\S', \S)$, then $\sp{\S'}{\S}$ splits
$\S'$ into at least two regions. Hence, the number of regions in $\P$
strictly increases.
However, since, $\N$ has finitely many nodes, the number of
regions in $\P$ is upper-bounded.

Next, we will argue that $\P$ that is returned is an
$\nn$-bisimulation.
After the $\sb{\S_i}$ operations, $\P$ only consists of regions which
agree on the activation functions and biases.
When the while loop terminates, there are no inconsistent pairs, that
is, the pre-sum condition of the bisimulation definition is satisfied.
Hence, the value of $\P$ when exiting the while loop is an $\nn$-bisimulation.

To show that $\P$ is the coarsest bisimulation, it remains to show
that $\P$ is coarser than any bisimulation of $\N$.
Let $\P'$ be any bisimulation of $\N$.
We will show that $\P'$ is finer than $\P$ at every stage of the algorithm.

Note that after exiting the for loop, $\P$ contains the maximal
groups which agree on both the activation functions and biases.
Every region of $\P'$ has to agree on the activation functions and
biases, since it is a bisimulation. So, every region of $\P'$ is
contained in some regions of $\P$, that is, $\P' \finer \P$.

Next, we show that $\P' \finer \P$ is an invariant for the while loop,
that is, if it holds at the beginning of the loop, then it also holds
at the end of the loop. So, when the while loop exits, we still have
$\P' \finer \P$.
More precisely, we need to show that if $\P' \finer \P$, then
replacing $\S'$ by $\sp{\S'}{\S}$ will still result in a partition
that is coarser than $\P'$.
In particular, we need to ensure that each region of $\P'$ that is
contained in $\S'$ is not split by the $\sp{\S'}{\S}$ operation.
Suppose a region $\S'' \subseteq \S'$ of $\P'$ is split, then there exists $\s''_1,
\s''_2 \in \S''$ such that $\ps{i}{\N}(\S, \s''_1) \not= \ps{i}{\N}(\S,
\s''_1)$.
But $\S$ is the disjoint union of some sets $\{\S''_1, \cdots, \S''_l\}$
of $\P'$.
Hence,  $\ps{i}{\N}(\S''_i, \s''_1) \not= \ps{i}{\N}(\S''_i,
\s''_1)$ for some $i$, since $\ps{i}{\N}(\S, \s'') = \sum_i
\ps{i}{\N}(\S''_i, \s'')$.
However, this contradicts the fact that $\P'$ is an $\nn$-bisimulation.
\end{proof}

Next, we present some complexity results on checking if a partition is
a bisimulation/$\delta$-bisimulation, complexity of constructing
reduced systems from a bisimulation/$\delta$-bisimulation and the
complexity of computing the coarsest bisimulation.

\begin{theorem}
  Given an $\nn$ $\N$, a partition $\P$ and $\epsilon \geq 0$,
checking if $\P$ is a bisimulation and checking if
$\P$ is an $\delta$-$\nn$-bisimulation both take time $\O(m)$, where
$m$ is the number of edges of $\N$.
Further, constructing $\abs{\N}{\P}$ for some $\N' \in \absa{\N}{\P}{\delta}$
takes time $\O(m)$ as well.
\end{theorem}
\begin{proof}
To check if $\P$ is a bisimulation, we can iterate over all the nodes
in a region to check if they have same activation function, bias,
and pre-sums with respect to every region of $\P_{i-1}$.
In doing so, we need to access each node and each edge at most once,
hence, the complexity is bounded by $\O(m)$.
For the $\delta$-bisimulation, we need to check if the biases and
pre-sums are within $\epsilon$.
We can compute the biases and pre-sums in one pass over the network in time
$\O(m)$ as before.
Then we can find the max and min values of the
bias/pre-sum values within each region, and check if the max and min
values are within $\epsilon$, this will take time $\O(m)$.
For constructing the reduced system, we need to find the activation
functions and biases of all the nodes in the reduced system, and the
weight on the edge between two groups. The total computation needs to
access each edge at most once.
%
%
\end{proof}


\begin{theorem}
  The minimization algorithm has a time complexity of
  $\O(\hat{n}(m+n\log n))$, where $n$ is the number of nodes and $m$ is
  the number of edges of $\N$, and $\hat{n}$ is the number of nodes in
  the minimized neural network. 
\end{theorem}
\begin{proof}
  $\sb{\S_i}$ needs to sort the elements in every group by the activation
  function/bias values, hence, takes time $\O(n \log n)$.
Finding an inconsistent pair takes the same time as checking whether
$\P$ is a bisimulation, that is,  $\O(m)$.
$\sp{\S'}{\S}$ take time at most $\O(m)$ to compute the pre-sums
and $\O(n \log n)$ to split.
Replacing $\S'$ by $\sp{\S'}{\S}$ takes time at most $\O(\hat{n})$ which is
upper-bounded by $\O(n)$.
So, each loop takes time $\O(m + n \log n)$.
The number of iterations of the while loop is upper bounded by the
number of regions in the minimized neural network, that is, $\O(\hat{n})$.
Hence, the minimization algorithm has a runtime of $\O(\hat{n}(m+n\log
n))$.
\end{proof}

\section{Conclusions}
We presented the notions of bisimulation and approximate bisimulation for
neural networks that provide semantic equivalence and semantic
closeness, respectively, and are applicable to neural networks with a
wide range of activation functions.
These provide foundational theoretical tools for exploring the
trade-off between the amount of reduction and the semantic deviation
in an approximation based verification paradigm for neural networks.
Our future work will focus on experimental analysis of this trade-off
on large scale neural networks.
The notions of bisimulation explored are syntactic in nature, and we
will explore semantic notions in the future.
We provide a minimization algorithm for finding the smallest $\nn$
that is bisimilar to a given neural network.
Though a unique minimal network does not exist with respect to
$\delta$-bisimulations, we will explore heuristics for constructing
small networks that are $\delta$-bisimilar.

\bibliographystyle{splncs04}
\bibliography{references}
\end{document}